\documentclass{article}
\usepackage{spconf,amsmath,graphicx}

\usepackage[utf8]{inputenc} 
\usepackage[T1]{fontenc}    
\usepackage{hyperref}       
\usepackage{url}            
\usepackage{booktabs}       
\usepackage{amsfonts}       
\usepackage{nicefrac}       
\usepackage{microtype}      
\usepackage{csvsimple}
\usepackage{amsopn}
\usepackage{amsthm}
\usepackage{etex}
\usepackage{graphicx}
\usepackage{endnotes}
\usepackage{hyperref}
\usepackage{epsfig,psfrag}
\usepackage{pst-all}
\usepackage{amssymb,amsfonts,upref,cite,epsf,color,bm}
\usepackage{graphicx}
\usepackage{color}
\usepackage{amsmath}
\usepackage{graphicx}
\usepackage{calc}
\usepackage{booktabs}
\usepackage{tikz}
\usepackage{pgfplots}
\newcommand\defeq{:=}

\usepackage{algorithm}
\usepackage{algpseudocode}
\floatname{algorithm}{Algorithm}
\algnewcommand\algorithmicinput{\textbf{Input:}}
\algnewcommand\INPUT{\item[\algorithmicinput]}
\algnewcommand\algorithmicoutput{\textbf{Output:}}
\algnewcommand\OUTPUT{\item[\algorithmicoutput]}
\newtheorem{assumption}{Assumption}

\DeclareMathOperator*{\argmax}{argmax}

\DeclareMathOperator*{\argmin}{arg\;min}

\newcommand\vect[1]{\mathbf #1}

\newcommand{\vx}{\vect{x}}  
\newcommand{\vy}{\vect{y}}  
\newcommand{\vz}{\vect{z}}

\newcommand{\mB}{\mathbf{B}}

\newcommand{\signalsize}{n}

\newcommand{\nDualLassoObj}{\mathcal{D}}

\newcommand{\graphsigs}{\mathbb{R}^{|\nodes|}} 
\newcommand{\edgesigs}{\mathbb{R}^{|\edges|}}

\newcommand{\edges}{\mathcal{E}}

\newcommand{\nLassoObj}{\mathcal{L}}
\newcommand{\cluster}{\mathcal{C}}
\newcommand{\nodes}{\mathcal{V}}
\newcommand{\graph}{\mathcal{G}}
\newcommand{\extgraph}{\widetilde{\mathcal{G}}}
\newcommand{\trainingset}{\mathcal{M}}
\newcommand{\samplingset}{\mathcal{M}}

\newcommand{\partition}{\mathcal{F}}

\newcommand{\numnodes}{n}

\newcommand{\incidencemtx}{B}

\newcommand{\primslp}{\widehat{\vx}}
\newcommand{\dualslp}{\widehat{\vy}}
\newcommand{\flowvec}{\widetilde{\mathbf{y}}}
\newcommand{\flow}{\tilde{y}}

\newtheorem{theorem}{Theorem}
\newtheorem{definition}[theorem]{Definition}

\newtheorem{proposition}[theorem]{Proposition}
\newtheorem{corollary}[theorem]{Corollary}
 
\usepackage{setspace}

\title{On the Duality between Network Flows and Network Lasso}

\name{Alexander Jung}
\address{Department of Computer Science, Aalto University, Finland; firstname.lastname(at)aalto.fi}

\begin{document}
	\maketitle
\begin{abstract}
Many applications generate data with an intrinsic network structure such as 
time series data, image data or social network data. The network Lasso (nLasso) 
has been proposed recently as a method for joint clustering and optimization 
of machine learning models for networked data. The nLasso extends the Lasso 
from sparse linear models to clustered graph signals. This paper explores the 
duality of nLasso and network flow optimization. We show that, in a very precise 
sense, nLasso is equivalent to a minimum-cost flow problem on the data network 
structure. Our main technical result is a concise characterization of nLasso solutions 
via existence of certain network flows. The main conceptual result is a useful link 
between nLasso methods and basic graph algorithms such as clustering or maximum flow. 
\end{abstract}

\section{Introduction}
\label{sec_intro}




The network Lasso (nLasso) has been proposed recently to fit localized models 
to networked data \cite{NetworkLasso}. Localized models allow to use different model 
parameters for different data nodes. However, the node-wise parameters are coupled 
by require them to have a small total variation (TV). 

Efficient methods to process networked data are offered by graph algorithms such 
as clustering or network flow optimization \cite{BertsekasNetworkOpt}. 
While these graph algorithm only use the network structure, the 
nLasso also takes additional information into account \cite{NetworkLasso}. 
We represent this additional information in the from of a graph signal which 
maps individual data points to a signal value (``label''). 

We explore the duality between nLasso and a minimum-cost flow problem. 
This is a special case of the duality between structured norm minimization 
and network flow problems studied in \cite{Mairal2010}. In contrast to \cite{Mairal2010}, 
we do not use this duality to apply network flow methods to solve nLasso but rather 
use the existence of certain network flows to characterize nLasso solutions. 

Our analysis relates the performance of nLasso methods for joint optimization 
and clustering to existence of network flows which serve as a proxy measure 
for the connectivity of clusters. This is somewhat similar to the concept of 
conductivities used for the design and analysis of clustering methods in \cite{Yin2017}. 
In contrast to \cite{Yin2017}, we study networked data point providing additional information 
in the form of a graph signal. 

The main contributions of this paper are as follows.  
\begin{itemize} 
\item We show that the convex dual of nLasso is equivalent to a particular minimum-cost flow problem. \\[-5mm]
\item We interpret a primal-dual method for nLasso as a distributed 
network flow optimization method.  \\[-5mm]
\item We characterize the solutions of nLasso via the existence of sufficient large network flows 
between cluster boundaries and sampled (labeled) nodes.  \\[-5mm]
\item We provide a novel interpretation of the nLasso parameter as a scaling of edge capacities in a flow network. 
\end{itemize}

{\bf Notation.} The sub-differential of a function $g(\vx)$ at $\vx_{0}\!\in\!\mathbb{R}^{n}$ is the set 
\begin{equation} 
\partial g(\vx_{0})\!\defeq\!\{ \vy\!\in\!\mathbb{R}^{n}\!:\!g(\vx)\!\geq\!g(\vx_{0})\!+\!\vy^{T}(\vx\!-\!\vx_{0}) \mbox{ for any } \vx \}. \nonumber 
\end{equation}  
The convex conjugate function of $g(\vx)$ is \cite{BoydConvexBook}
\begin{equation}
\label{equ_def_convex_conjugate}
g^{*}(\hat{\vy}) \defeq \sup_{\vy \in \mathbb{R}^{n}} \vy^{T}\hat{\vy}- g( \vy). 
\end{equation}

\vspace*{-3mm}
\section{Recovering Clustered Graph Signals}
\label{sec_setup}
We consider networked data whose network structure is encoded in an undirected empirical 
graph $\graph = (\nodes, \edges, \mathbf{W})$. 
The nodes $i \in \nodes=\{1,\ldots,\numnodes\}$ of the empirical graph 
represented individual data points. Similar data points are connected by 
an edge $\{i,j\} \in \edges$ with some weight $W_{i,j}\!>\!0$ that quantifies 
the amount of similarity between $i,j \in \nodes$. We depict an example of an 
empirical graph $\graph$ in Fig.\ \ref{fig:duality}. 

\begin{figure}[htbp]
	\includegraphics[width=\columnwidth]{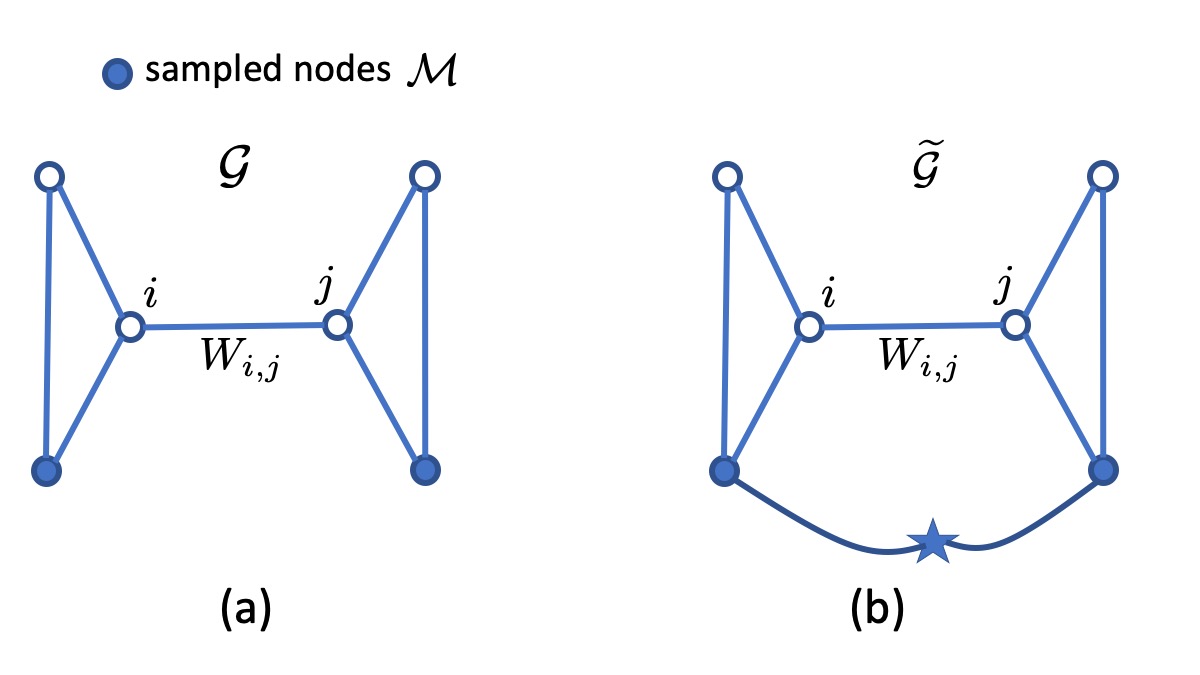}
	\vspace*{-8mm}
	\caption{(a) Empirical graph of networked data including a sampling set of nodes whose signal values are observed. 
		(b) Extended empirical graph obtained by adding the node $\star$ and edges between sampled nodes $i \in \samplingset$ and $\star$. } \label{fig:duality}
\end{figure}

The neighbourhood $\mathcal{N}(i)$ and degree $d_{i}$ 
of a node $i \in \nodes$ are defined, respectively, as 
\vspace*{-2mm}
\begin{equation} 
\label{equ_def_neighborhood}
\mathcal{N}(i) \defeq \{ j \in \nodes : \{i,j\} \!\in\!\edges \} \mbox{, } d_{i} \defeq \big|  \mathcal{N}(i)\big|. 
\vspace*{-2mm}
\end{equation} 

For a given undirected empirical graph $\graph=(\nodes,\edges,\mathbf{W})$, 
we orient the undirected edge $\{i,j\}$ by defining the head as $e^{+}\!=\!\min \{i,j\}$ 
and the tail as $e^{-}\!=\!\max \{i,j\}$. Each undirected edge $e\!=\!\{i,j\}$ is associated 
with the directed edge $(e^{+},e^{-})$. 

We need the directed neighbourhoods $\mathcal{N}^{+}(i)\!=\!\{ j\!\in\!\nodes: (i,j)\!\in\!\edges \}$, 
$\mathcal{N}^{-}(i)\!=\!\{ j\!\in\!\nodes: (j,i)\!\in\!\edges \}$ and the incidence matrix $\mB \!\in\! \mathbb{R}^{\edges \times \signalsize}$, 
\vspace*{-2mm}
\begin{equation}
\incidencemtx_{e,i}\!=\!1 \mbox{ for } i\!=\!e^{+}, \incidencemtx_{e,i}\!=\!-1 \mbox{ for }i\!=\!e^{-},\incidencemtx_{e,i}\!=\!0 \mbox{ else.}    \label{equ_def_incidence_mtx}
\end{equation} 

Beside the network structure, datasets carry additional information which we represent 
by a graph signal $\vx = (x_{1},\ldots,x_{\signalsize})^{T} \in \mathbb{R}^{\signalsize}$. 
The graph signal values $x_{i}$ might represent instantaneous 
amplitudes of an audio signal, the greyscale values of image pixels or the probabilities 
of social network members taking a particular action. 
We assume that signal values $x_{i}$ are known only at few nodes $i \in \nodes$ of a (small) sampling  
set $\trainingset \subseteq \nodes$. Our goal is to recover the unknown signal values $x_{i}$ for $i \in \nodes \setminus \trainingset$. 

To recover the signal values $x_{i}$, for $i\!\in\!\nodes$ based on knowing them 
only on a (small) training set, we exploit the tendency 
of natural graph signals to be clustered.
\vspace*{-1mm}
\begin{assumption}[informal]
\label{equ_asspt_cluster_informal}
Nodes $i,j\!\in\!\nodes$ within a well-connected subset (cluster) 
have similar signal values $x_{i}\!\approx\!x_{j}$. 
\vspace*{-3mm}
\end{assumption} 
This clustering assumption is used in image processing to model images that 
are composed of few components within which the pixel colours are approximately 
constant \cite{rudin1992nonlinear}. The clustering assumption is also used in social 
sciences where signal values represent certain features of individuals that are similar 
within well-connected groups (clusters) \cite{NewmannBook}. More broadly, (variants) of 
the clustering assumption motivate semi-supervising learning methods \cite{SemiSupervisedBook}. 

To make the informal Asspt.\ \ref{equ_asspt_cluster_informal} precise we need a measure for 
how well a graph signal conforms to the cluster structure of the empirical graph $\graph$. 
We measure this ``clusteredness'' of a graph signal $\vx$ using the weighted TV \cite{rudin1992nonlinear,Wang2016}
\begin{equation} 
\label{equ_def_TV}
\| \vx \|_{\rm TV} \defeq \sum_{\{i,j\} \in \edges} W_{i,j}  | x_{j}\!-\!x_{i}|.
\end{equation} 

Signal recovery methods based on TV minimization \eqref{equ_def_TV} turn out 
to be attractive statistically and computationally. These methods allows to recover 
clustered graph signals from very few signal samples \cite{AvilesRivero2019,elalaoui16,NSZ09}. 
This property is appealing for applications where the acquisition of signal values 
(labels) is costly. Computationally, TV minimization can be implemented as highly 
scalable message passing protocols (see Section \ref{sec_lNLasso_pd}). 

As shown in Sec.\ \ref{sec_nlasso_and_dual}, TV minimization  
is, in a very precise sense, equivalent to optimizing network flows. 
The concept of network flows is somewhat dual to the concept of graph signals. 
While the domain of graph signals is the node set $\nodes$ of a graph, network 
flows are defined on the edges $\edges$ of a graph. 
\begin{definition}
	A network flow $\flow:\!\edges\!\rightarrow\!\mathbb{R}$ with supplies $v_{i}$,
	assigns each directed edge $e=(i,j)\!\in\!\edges$ the value $\flow_{e}$ with
	\begin{itemize} 
		\item the capacity constraints: 
		\begin{equation}
		\label{equ_def_cap_constraints}
		|\flow_{e}| \leq \lambda W_{e} \mbox{ for each edge } e \in \edges,
		\end{equation} 
		\item and the conservation law: 
		\begin{equation} 
		\label{equ_conservation_law}
		\hspace*{-3mm}\sum_{j \in \mathcal{N}^{+}(i)} \hspace*{-2mm}\flow_{(i,j)}\!-\!\sum_{j \in \mathcal{N}^{-}(i)} \hspace*{-2mm} \flow_{(j,i)} = v_{i} \mbox{ for each } i\!\in\!\nodes.
		\end{equation} 
	\end{itemize} 
\end{definition}


\vspace*{-2mm}
\section{Network Lasso and its Dual}
\label{sec_nlasso_and_dual}
\vspace*{-2mm}

The cluster assumption suggests to learn graph signals by balancing 
the empirical error with the TV $\| \tilde{\vx} \|_{\rm TV}$, 
\begin{align} 
\primslp &\!\in\!\argmin_{\tilde{\vx} \in \graphsigs} (1/2) \sum_{i \in \samplingset} (x_{i}\!-\!\tilde{x}_{i})^2\!+\!\lambda  \| \tilde{\vx} \|_{\rm TV}.  \label{equ_nLasso}
\end{align}
The optimization problem \eqref{equ_nLasso} is a special case of 
the original (generic) nLasso formulation \cite{NetworkLasso}. 
Since the objective function and the constraints in \eqref{equ_nLasso} are 
convex, the optimization problem \eqref{equ_nLasso} is a convex 
optimization problem \cite{BoydConvexBook}. 

The nLasso \eqref{equ_nLasso} implements regularized risk 
minimization using TV \eqref{equ_def_TV} as regularization term \cite{VapnikBook}. 
The solutions $\hat{\vx}$ of \eqref{equ_nLasso} make an optimal compromise between 
consistency with observed signal samples $x_{i}$, for $i \in \samplingset$, and small TV $\| \tilde{\vx} \|_{\rm TV}$. 
The tuning parameter $\lambda\!>\!0$ in \eqref{equ_nLasso} allows to trade a small mean squared error (MSE) 
$(1/2) \sum_{i \in \trainingset}  (\hat{x}_{i}\!-\!x_{i})^2$ against a small TV $\|\hat{\vx}\|_{\rm TV}$ 
of the recovered graph signal $\hat{\vx}$. A large $\lambda$ enforces 
small TV, while a small $\lambda$ favours low MSE. 

The non-smooth objective function in \eqref{equ_nLasso} rules out 
gradient (descent) methods.  However, the objective function is the 
sum of two function that can be efficiently minimized individually. 
This compositional structure of \eqref{equ_nLasso} can be exploited 
by defining a dual problem. 

It turns out that this dual problem has an interpretation as network (flow) 
optimization \cite{BertsekasNetworkOpt}. Moreover, by jointly considering 
\eqref{equ_nLasso} and its dual, we obtain an efficient method for 
simultaneously solving both problems (see Section \ref{sec_lNLasso_pd}). 

To define the dual problem we first rewrite nLasso \eqref{equ_nLasso} 
as 
\begin{align}
\label{equ_nLasso_unconstr}
\primslp & \!\in\! \argmin_{\tilde{\vx} \in \mathbb{R}^{\signalsize}} \nLassoObj(\tilde{\vx}) \defeq g(\mB \tilde{\vx}) + h(\tilde{\vx}),
\end{align}
with the incidence matrix $\mB$ (see \eqref{equ_def_incidence_mtx}) and 
\begin{equation}
 g(\tilde{\vy})\!\defeq\!\sum_{e \in \edges} 
 \lambda W_{e}|y_{e}| \mbox{, and } h(\tilde{\vx})\!\defeq\!(1/2) \sum_{i \in \samplingset} (\tilde{x}_{i}\!-\!x_{i})^{2}. \label{equ_def_components_primal}
 \end{equation} 

We refer to \eqref{equ_nLasso_unconstr} as the primal problem (or formulation) of nLasso \eqref{equ_nLasso}. 
The dual problem is 
\begin{align} 
\label{equ_dual_SLP}
\dualslp & \!\in\! \argmax_{\vy \in \mathbb{R}^{|\edges|}} \nDualLassoObj(\vy) \defeq -h^{*}(-\mB^{T} \vy) - g^{*}(\vy).
\end{align} 
The objective function $\nDualLassoObj(\vy)$ of the dual problem \eqref{equ_dual_SLP} is composed of the convex conjugates 
(see \eqref{equ_def_convex_conjugate}) of the components $h(\vx)$ and $g(\vy)$ of the primal problem \eqref{equ_nLasso_unconstr}. 
These convex conjugates are given explicitly by 
\begin{align} 
\label{equ_conv_conj_h}
h^{*} (\tilde{\vx}) & = \sup_{\vz \in \mathbb{R}^{\signalsize}}  \vz^{T} \tilde{\vx}- h(\vz)  \\
& \hspace*{-12mm} \stackrel{\eqref{equ_def_components_primal}}{=} \begin{cases} \infty &\mbox{ if } \tilde{x}_{i}\!\neq\!0 \mbox{ for some } i\!\in\!\nodes\!\setminus\!\trainingset \\ (1/2) \!\sum\limits_{i \in \trainingset}\!  \tilde{x}_{i} x_{i}\!+\!\tilde{x}_{i}^{2}/2& \mbox{ otherwise,} \end{cases}. \nonumber 
\end{align} 
and 
\begin{align} 
\label{equ_conv_conj_g}
g^{*}(\vy)& = \sup_{\vz \in \mathbb{R}^{\edges}}  \vz^{T} \vy - g(\vz)  \stackrel{\eqref{equ_def_components_primal}}{=} \sup_{\vz \in \mathbb{R}^{\edges}}  \vz^{T} \vy - \lambda\sum_{e \in \edges} W_{e} |z_{e}| \nonumber  \\
& \stackrel{}{=} \begin{cases}  \infty & \mbox{ if } |y_e|> \lambda W_{e} \mbox{ for some } e \in \edges \\ 0 & \mbox{ otherwise.} \end{cases}
\end{align} 

The relation between the primal problem \eqref{equ_nLasso_unconstr} and the dual 
problem \eqref{equ_dual_SLP} is made precise by \cite[Thm.\ 31.3]{RockafellarBook}. 
In particular, the optimal values of \eqref{equ_nLasso_unconstr} and \eqref{equ_dual_SLP} coincide: 
\begin{equation} 
\label{equ_zero_duality_gap}
\min_{\tilde{\vx} \in \mathbb{R}^{\signalsize}}  \underbrace{g(\mB \tilde{\vx}) + h(\tilde{\vx})}_{\nLassoObj(\tilde{\vx})} = \max_{\vy \in \mathbb{R}^{\edges}} \underbrace{-h^{*}(-\mB^{T}\vy) - g^{*}(\vy)}_{\nDualLassoObj(\vy)}. 
\end{equation}
According to \cite[Thm.\ 31.3]{RockafellarBook}, a pair of vectors $\widehat{\vx} \in \graphsigs, \dualslp \in \mathbb{R}^{\edges}$ 
are solutions to the primal \eqref{equ_nLasso_unconstr} and dual problem \eqref{equ_dual_SLP}, respectively, if and only if 
\begin{equation}
\label{equ_two_coupled_conditions}
-(\mB^{T} \dualslp) \in \partial h(\primslp) \mbox{, } \mB \primslp \in \partial g^{*}(\dualslp) . 
\end{equation} 
Given any dual solution $\dualslp \in \mathbb{R}^{\edges}$ to \eqref{equ_dual_SLP}, 
every nLasso solution $\hat{\vx}$ must satisfy \eqref{equ_two_coupled_conditions}. 
The condition \eqref{equ_two_coupled_conditions} also motivates a primal-dual 
method to solve \eqref{equ_nLasso} (see Section \ref{sec_lNLasso_pd}).

Our main result is the equivalence of the nLasso dual \eqref{equ_dual_SLP} to a 
minimum-cost flow problem for an extended empirical graph $\extgraph$ (see Fig.\ \ref{fig:duality}-(b)). 
The graph $\extgraph$ is obtained from the empirical graph $\graph$ by adding  
the node $\star$ and the edges $ (i,\star)$ for each sampled node $i\!\in\!\samplingset$. 


\begin{proposition}
\label{prop_dual_TV_min_flow}
The dual problem \eqref{equ_dual_SLP} of nLasso \eqref{equ_nLasso} 
is equivalent to the minimum-cost flow problem  
\begin{align}
\label{equ_network_flow}
\min_{\flowvec \in \mathbb{R}^{\widetilde{\edges}}} & \sum_{i \in \trainingset}  \flow_{(i,\star)} \big( (1/2) \flow_{(i,\star)}\!-\!x_{i}  \big),   \\ 
\mbox{s.t.} & \sum_{j \in \mathcal{N}^{+}(i)} \hspace*{-2mm}\flow_{(i,j)}\!-\!\sum_{j \in \mathcal{N}^{-}(i)} \hspace*{-2mm} \flow_{(j,i)}  = 0 \mbox{ for all } i \in \nodes \cup \{\star\} \label{equ_constrains_extended_graph} \\ 
&  |\flow_{e} | \leq \lambda W_{e} \mbox{ for all } e \in \edges. \nonumber
\end{align}
\end{proposition} 
The capacity constraints in \eqref{equ_constrains_extended_graph} do not 
include the augmented edges $(i,\star)$ for $i\!\in\!\samplingset$. The role 
of the nLasso parameter $\lambda$ in \eqref{equ_constrains_extended_graph} 
is a scaling of the edge capacities $W_{i,j}$. 

The problem \eqref{equ_network_flow} is an instance of a minimum-cost flow problem with 
convex separable cost functions (see \cite[Ch.\ 8]{BertsekasNetworkOpt}). Efficient methods 
for such flow problems are presented in \cite{BertsekasNetworkOpt}. The special case of a 
minimum-cost flow problem with convex quadratic functions, such as in \eqref{equ_network_flow}, 
is studied in \cite{Minoux1984,Vegh2016}. Instead of applying network flow methods, we will directly 
solve nLasso using a primal-dual method (see Section \ref{sec_lNLasso_pd}). 

\vspace*{-2mm}
\section{Statistical Aspects} 
\vspace*{-2mm}
\label{sec_stat_aspects}
To study nLasso solutions \eqref{equ_nLasso}, we will use a 
simple but useful model which implements the cluster assumption Asspt.\ \ref{equ_asspt_cluster_informal}, 
\vspace*{-2mm}
\begin{equation} 
\label{equ_piecewise_constant}
x_{i} = c_{k} \mbox{ for all nodes } i \in \cluster_{k} \mbox{ with coefficients } c_{k} \in \mathbb{R}. 
\vspace*{-1mm}
\end{equation} 
The signal model \eqref{equ_piecewise_constant} involves an arbitrary 
but fixed partition $\partition = \big\{ \cluster_{1},\ldots,\cluster_{F} \big\}$ 
of the nodes $\nodes$ into disjoint clusters. Piece-wise constant signals 
are a special case of the large class of piece-wise polynomial graph signals \cite{ChenClustered2016}. 

Combining Proposition \ref{prop_dual_TV_min_flow} with the optimality 
condition (see \cite[Prop.\ 8.2.]{BertsekasNetworkOpt} offers a 
concise characterization of nLasso \eqref{equ_nLasso} solutions via 
existence of certain network flows. 
\begin{corollary}
\label{cor_flow_satur_constant}
Consider a graph signal \eqref{equ_piecewise_constant} 
and a flow $\flow_{e}$ on $\extgraph$ satisfying \eqref{equ_constrains_extended_graph} and 
\vspace*{-1mm}
\begin{align} 
| \flow_{e} | \begin{cases} &= \lambda W_{e} \mbox{ for } e \in \partial \partition,\mbox{ and }\\[-1mm]
           &< \lambda W_{e} \mbox{ otherwise. } \end{cases}  \label{equ_sat_non_sat}
\end{align}
The flow $\flow_{e}$ solves \eqref{equ_network_flow} if and only if, for each cluster $\cluster_{k}$, 
\vspace*{-2mm}
\begin{equation}
\label{equ_balanced_supply}
x_{i}  - \flow_{(i,\star)} = x_{j}  - \flow_{(j,\star)} \mbox{ for any } i,j \in \cluster_{k} \cap \samplingset. 
\vspace*{-1mm}
\end{equation} 
Given a flow $\flow_e$ satisfying \eqref{equ_sat_non_sat}, \eqref{equ_balanced_supply}, 
any nLasso \eqref{equ_nLasso} solution $\hat{\vx}$ is constant over non-saturated edges,
\begin{equation} 
\label{equ_opt_condition_constant}
\hat{x}_{i} = \hat{x}_{j} \mbox{ for } (i,j) \in \edges \mbox{ with } |\flow_{(i,j)}| < \lambda W_{i,j}. 
\end{equation}
Moreover, for each $j\in \cluster_{k}$, the nLasso value is
\begin{equation}
\hat{x}_{j}\!=\!x_{i}\!-\!\big[  \hspace*{-3mm}\sum_{i' \in \mathcal{N}^{+}(i)} \hspace*{-3mm} \flow_{(i,i')}- \hspace*{-3mm} \sum_{i' \in \mathcal{N}^{-}(i)} 
\hspace*{-4mm} \flow_{(i',i)} \big]  \mbox{ for some } i\!\in\!\samplingset \cap \cluster_{k}.  \label{equ_nLasso_value_solution}
\end{equation} 
\end{corollary}

\begin{proof}
The optimality condition \cite[Prop.\ 8.2.]{BertsekasNetworkOpt} reveals that 
conditions \eqref{equ_sat_non_sat} and \eqref{equ_balanced_supply} are necessary and sufficient 
for the flow $\flow$ to be a solution to the minimum-cost flow problem \eqref{equ_network_flow}. 

Consider a flow $\flow$ that satisfies \eqref{equ_sat_non_sat} and \eqref{equ_balanced_supply} 
and therefore solves \eqref{equ_network_flow}. By Theorem \ref{prop_dual_TV_min_flow}, we can 
obtain a solution $\hat{y}$ to the nLasso dual problem \eqref{equ_dual_SLP} by $\hat{y}_{e} \defeq \flow_{e}$ for each 
edge $e \in \edges$. 

Given this particular (optimal) dual solution $\widehat{\vy}$, any solution $\hat{\vx}$ to TV minimization 
has to satisfy \eqref{equ_two_coupled_conditions}. Combining \eqref{equ_two_coupled_conditions} 
with properties of the sub-differential $\partial g^{*}(\vy)$ (see \eqref{equ_conv_conj_g} and \cite[Sec.\ 32]{RockafellarBook}) 
yields \eqref{equ_opt_condition_constant}. 
\vspace*{-3mm}
\end{proof}
The usefulness of Prop.\ \ref{cor_flow_satur_constant} depends on the ability to 
construct flows on $\extgraph$ satisfying \eqref{equ_sat_non_sat} and \eqref{equ_balanced_supply}. 
This might be easy for simple graph structures such as chains (see Sec.\ \ref{sec_num_experiment}). 
Another option is to use tractable probabilistic models, such as stochastic block models, 
for the empirical graph \cite{Mossel2012}. A large deviation analysis allows then to obtain 
characterization of network flows that hold with high probability \cite{Karger1999}. 

The condition \eqref{equ_sat_non_sat} can be used to guide the choice of the 
nLasso parameter $\lambda$ (see \eqref{equ_nLasso}). Using a larger value $\lambda$ will typically 
make condition \eqref{cor_flow_satur_constant} more likely to be satisfied, if the clusters 
$\cluster_{k}$ are sufficiently well connected. However, larger values of $\lambda$ will result 
in a bias of the nLasso estimate $\hat{x}_{i}$ due to \eqref{equ_nLasso_value_solution}. Thus, 
condition \eqref{cor_flow_satur_constant} and \eqref{equ_nLasso_value_solution} can help to 
avoid choosing $\lambda$ neither too small (which would make \eqref{cor_flow_satur_constant} 
unlikely to hold) nor too large (which would imply a large bias via \eqref{equ_nLasso_value_solution}). 


\vspace*{-4mm}
\section{Computational Aspects} 
\vspace*{-2mm}
\label{sec_lNLasso_pd}
The characterization \eqref{equ_two_coupled_conditions} of solutions to the 
nLasso \eqref{equ_nLasso} and its dual suggests to apply a convex 
primal-dual method \cite{pock_chambolle}. The implementation of this method 
follows similar lines as in \cite{LocalizedLinReg2019} and results in the iterations  
\begin{align} 
\tilde{x}_{i}  & \!\defeq\! 2 \hat{x}^{(k)}_{i} - \hat{x}^{(k\!-\!1)}_{i} \mbox{ for } i\!\in\!\nodes \label{equ_pd_first_update} \\[2mm]
\hat{y}^{(k\!+\!1)}_{e} &\!\defeq\!\hat{y}^{(k)}\!+\! (1/2)  (\tilde{x}_{i}\!-\!\tilde{x}_{j})\mbox{ for } e=(i,j)\!\in\!\edges \label{equ_pd_two}  \\[2mm]
\hat{y}^{(k\!+\!1)}_{e} &\!\defeq\! \hat{y}_{e}^{(k\!+\!1)} \max\{1, |\hat{y}_{e}^{(k\!+\!1)}|/(\lambda W_{e}) \}  \mbox{ for } (i,j)\!\in\!\edges   \label{equ_pd_three}  \\[2mm]
\hat{x}^{(k\!+\!1)}_{i}  & \!\defeq\! \hat{x}^{(k)}_{i}\!-\!\gamma_{i} \big[\hspace*{-1mm}\sum_{j\!\in\!\mathcal{N}^{+}(i)} \hspace*{-3mm}\hat{y}^{(k\!+\!1)}_{(i,j)}  \!-\!\hspace*{-1mm}\sum_{j\!\in\!\mathcal{N}^{-}(i)} \hspace*{-3mm}\hat{y}^{(k\!+\!1)}_{(j,i)} \big] \mbox{ for }  i\!\in\!\nodes \label{equ_pd_four} \\[2mm]
\hat{x}_{i}^{(k\!+\!1)} &\!\defeq\!  \big(\gamma_{i} x_{i}\!+\!\hat{x}^{(k\!+\!1)}_{i}\big)/(\gamma_{i}\!+\!1) \mbox{ for every } i\!\in\!\trainingset  \label{equ_pd_five} \\[2mm]
\bar{x}^{(k)}_{i} & \!\defeq\! (1\!-\!1/k) \bar{x}^{(k\!-\!1)}_{i}\!+\!(1/k) \hat{x}^{(k)}_{i} \mbox{ for } i\!\in\!\nodes. \label{equ_pd_last_update}
\end{align}
Here, $k=0,1,\ldots$ denotes the iteration counter and $\gamma_{i} \defeq 1/d_{i}$ 
is the inverse node degree \eqref{equ_def_neighborhood}. 

Standard results on convergence of primal-dual methods ensure that, irrespective 
of the initializations $\hat{x}^{(0)}$ and $\hat{y}^{(0)}$, the iterates $\bar{x}^{(k)}$ 
converge to a solution of nLasso \eqref{equ_nLasso} \cite{pock_chambolle}. Moreover, the 
rate at which the sub-optimality, in terms of objective value, decreases is $1/k$. 
This rate is essentially optimal, in a worst-case sense, for message passing methods \cite{ComplexitySLP2018}. 

We can interpret the updates \eqref{equ_pd_first_update}-\eqref{equ_pd_last_update} as a 
message passing rules for network-flow optimization \cite{}. In particular, 
the iterate  $\hat{y}^{(k)}$ is a flow which tends to a solution $\hat{y}$ of the 
nLasso dual problem \eqref{equ_dual_SLP}. 

The update \eqref{equ_pd_three} aims at enforcing the capacity constraints \eqref{equ_def_cap_constraints} 
for the flow iterate $\hat{y}^{(k)}$. The update \eqref{equ_pd_four} amounts to adjusting the current 
nLasso estimate $\hat{x}_{i}^{(k)}$, for each node $i \in\!\nodes$ by the 
demand induced by the current flow $\hat{y}^{(k)}$. Thus, \eqref{equ_pd_four} enforces 
the conservation law \eqref{equ_conservation_law} with demands $v_{i} = \hat{x}_{i}^{(k)}$. 

For each unobserved node $i\!\in\!\nodes \setminus \trainingset$, we can interpret the iterate $\hat{x}_{i}^{(k)}$ 
as the (scaled) cumulative demand induced by the flows $\hat{y}^{(k')}$ for $k'=1,\ldots,k$. 
The labeled nodes $i \in \trainingset$ have a constant supply $\hat{x}^{(k)}_{i} = x_{i}$ whose 
amount is the label $x_{i}$. The update \eqref{equ_pd_two} balances the discrepancies between 
the cumulated demands $\hat{x}^{(k)}_{i}$ by adapting the flow $\hat{y}^{(k)}_{(i,j)}$ through 
an edge $e=(i,j) \in \edges$ according to the difference $(\tilde{x}_{i} - \tilde{x}_{j})$. 


\vspace*{-4mm}
\section{Numerical Experiments} 
\vspace*{-3mm}
\label{sec_num_experiment}
We verify Prop.\ \ref{cor_flow_satur_constant} numerically using a chain-structured 
empirical graph $\graph$ which might represent time series data \cite{Brockwell91}. 
The chain structured empirical graph $\graph$ contains $\numnodes=10$ nodes which 
are partitioned into two clusters $\cluster_{1}\!=\!\{1,\ldots,5\}$ and $\cluster_{2}\!=\!\{6,\ldots,10\}$. 
Intra-cluster edges $e$ connecting nodes within the same cluster have unit weight $W_{e}\!=\!1$, 
while the boundary edge $e\!=\!\{5,6\}$ has weight $W_{e}\!=\!1/4$. 

We iterate the updates \eqref{equ_pd_first_update}-\eqref{equ_pd_last_update} for a fixed 
number of $K=1000$ iterations to recover a piece-wise constant graph signal \eqref{equ_piecewise_constant}, with 
$c_{1}\!=\!1$, $c_{2}\!=\!0$, from its values on the sampling set $\samplingset\!=\!\{2,7\}$. The nLasso 
parameter was set to $\lambda\!=\!1$ (see \eqref{equ_nLasso}). 
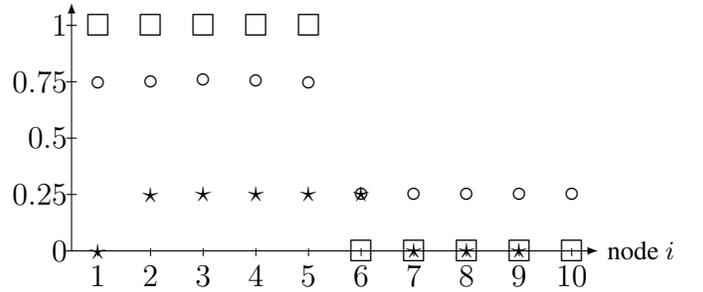
\begin{figure}[htbp]
	\begin{center}
		\begin{tikzpicture}
		\tikzset{x=0.7cm,y=3cm,every path/.style={>=latex},node style/.style={circle,draw}}
		\csvreader[ head to column names,%
		late after head=\xdef\iold{\i}\xdef\xold{\x},,%
		after line=\xdef\iold{\i}\xdef\xold{\x}]%
		{NumExpChainPrimal.csv}{}
		{\draw [line width=0.0mm] (\iold, \xold) (\i,\x) node {\large $\circ$};
		}
			\csvreader[ head to column names,%
	late after head=\xdef\iold{\i}\xdef\xold{\x},,%
	after line=\xdef\iold{\i}\xdef\xold{\x}]%
	{NumExpChainSig.csv}{}
	{\draw [line width=0.0mm] (\iold, \xold) (\i,\x) node {\large $\square$};
	}
	\csvreader[ head to column names,%
late after head=\xdef\iold{\i}\xdef\yold{\y},,%
after line=\xdef\iold{\i}\xdef\yold{\y}]%
{NumExpChainDual.csv}{}
{\draw [line width=0.0mm] (\iold, \yold) (\i,\y) node {\large $\star$};
}	

		\draw[->] (0.4,0) -- (10.5,0);
		
		\node [right] at (10.5,0.0) {\centering node $i$};
		\draw[->] (0.5,0) -- (0.5,1.1);
		
		\foreach \label/\labelval in {0/$0$,0.25/$0.25$,0.5/$0.5$,0.75/$0.75$,1/$1$}
		{ 
			\draw (0.4,\label) -- (0.6,\label) node[left] {\large \labelval};
		}
		
		\foreach \label/\labelval in {1/$1$,2/$2$,3/$3$,4/$4$,5/$5$,6/$6$,7/$7$,8/$8$,9/$9$,10/$10$}
		{ 
			\draw (\label,1pt) -- (\label,-2pt) node[below] {\large \labelval};
		}
		\end{tikzpicture}
		\vspace*{-10mm}
	\end{center}
	\caption{Piece-wise constant graph signal $x_{i}$ (``$\square$''), nLasso iterate $\hat{x}^{(K)}_{i}$ (``$\circ$'') 
		and dual iterate $\hat{y}^{(K)}_{(i,i+1)}$ (``$\star$''). 
	}
	\label{fig_solution_nLasso_chain}
	\vspace*{-3mm}
\end{figure}
To construct a flow $\hat{y}_{e}$, for edges $e=(i,i+1)$, that solves 
the nLasso dual \eqref{equ_network_flow}, we have to ensure \eqref{equ_sat_non_sat} 
and \eqref{equ_balanced_supply}. The condition \eqref{equ_sat_non_sat} is ensured by 
choosing $\hat{y}_{e}\!=\!\lambda W_{e}=1/4$ for the boundary edge $e\!=\!\{5,6\}$. The 
condition \eqref{equ_balanced_supply} is trivially satisfied since each cluster contains only 
one single sampled node. The flow values $\hat{y}_{e}$ of intra-cluster edges $(i,i+1)$, with 
$i\!\neq\!5$, can be determined using the conservation law \eqref{equ_constrains_extended_graph} for 
each node $i\!\notin\!\samplingset$ outside the sampling set. This results in 
\vspace*{-3mm}
\begin{equation}
\hat{y}_{(i,i+1)}= \begin{cases} 1/4 & \mbox{ for } i=2,\ldots,6 \\ 0 & \mbox{ otherwise.} \end{cases}. 
\vspace{-3mm}
\end{equation}
which resembles the iterate $y^{(K)}_{e}$ in Fig.\ \ref{fig_solution_nLasso_chain}. 
Given the dual solution $\hat{y}_{e}$, we obtain an nLasso solution $\hat{x}$ via \eqref{equ_opt_condition_constant}, \eqref{equ_nLasso_value_solution}. The simulation source code is available at \url{https://github.com/alexjungaalto/ResearchPublic/tree/master/FlowsNLasso}. 

\vspace*{-3mm}
\section{Conclusion}
\vspace{-3mm}

We have developed a theory of duality between nLasso and a minimum-cost 
network flow problem. This duality has been used to characterize nLasso 
solutions via the existence of certain network flows. Our work opens up 
several interesting research directions. It is interesting to study how 
parametric flow algorithms could be used to efficiently compute entire 
nLasso solution paths for varying $\lambda$ in \eqref{equ_nLasso}. 
Understanding the behavior of nLasso in terms of network flows could help to 
guide model reduction techniques by sparsifying the empirical graph without scarifying nLasso 
accuracy \cite{Batson2013}.

\bibliographystyle{plain}
\bibliography{/Users/alexanderjung/Literature}
\newpage
\section{Supplementary Material}

To derive \eqref{equ_pd_first_update}-\eqref{equ_pd_last_update} as a fixed point iteration 
based on the optimality condition \eqref{equ_two_coupled_conditions}, 
we rewrite \eqref{equ_two_coupled_conditions} as 
\begin{align}
\primslp - {\bm \Gamma} \mB^{T} \dualslp & \in \primslp+ {\bm \Gamma} \partial h(\primslp)  \nonumber \\ 
2 {\bm \Lambda}  \mB \primslp +\dualslp  & \in {\bm \Lambda} \partial g^{*}(\dualslp)+ {\bm \Lambda} \mB\primslp+\dualslp,\label{equ_manipulated_coupled_conditions} 
\end{align}
with the invertible diagonal matrices
\begin{align} 
{\bf \Lambda}\!\defeq\! (1/2) \mathbf{I}\!\in\!\mathbb{R}^{\edges\!\times\!\edges} \mbox{, }
{\bf \Gamma}\!\defeq\!{\rm diag} \{ \gamma_{i}\!=\!1/d_{i} \}_{i\!=\!1}^{\signalsize}\!\in\!\mathbb{R}^{\signalsize\!\times \signalsize}.\label{equ_def_scaling_matrices}
\end{align}
The particular choice \eqref{equ_def_scaling_matrices} ensures that \cite[Lemma 2]{PrecPockChambolle2011}
\begin{equation}
\| {\bf \Gamma}^{1/2} \mB^{T} {\bf \Lambda}^{1/2} \|_{2} < 1, \nonumber
\end{equation}
which ensure converge of the proposed method. There are other choices 
than \eqref{equ_def_scaling_matrices} that ensure convergence. Data-driven 
tuning of the matrices ${\bf \Gamma}, {\bf \Lambda}$ is beyond the scope of this paper. 

We further develop the characterization \eqref{equ_manipulated_coupled_conditions} using 
the resolvent operators for the (set-valued) operators ${\bf \Lambda}  \partial g^{*}(\vy)$ and 
${\bm \Gamma} \partial h(\vx)$ (see \eqref{equ_nLasso_unconstr} and  \cite[Sec. 1.1.]{PrecPockChambolle2011}),  
\begin{align}
(\mathbf{I}\!+\!{\bf \Lambda} \partial g^{*})^{-1} (\tilde{\vy}) & \!\defeq\! \argmin\limits_{\vz \in \edgesigs} g^{*}(\vz)\!+\!(1/2) \| \tilde{\vy}\!-\!\vz \|_{{\bm \Lambda}^{-1}}^{2} \nonumber \\ 
(\mathbf{I}\!+\!{\bm \Gamma} \partial h)^{-1} (\tilde{\vx}) & \!\defeq\! \argmin\limits_{\vz \in \graphsigs} h(\vz)\!+\!(1/2) \| \tilde{\vx}\!-\!\vz\|_{{\bm \Gamma}^{-1}}^{2}. \label{equ_iterations_number_112}
\end{align}

Applying \cite[Prop. 23.2]{Bauschke:2017} and \cite[Prop. 16.44]{Bauschke:2017} to the optimality condition \eqref{equ_manipulated_coupled_conditions} 
yields the equivalent condition (for $\primslp$, $\dualslp$ to be primal and dual optimal) 
\begin{align}
\primslp &= (\mathbf{I}\!+\!{\bm \Gamma} \partial h)^{-1} (\primslp\!-\!{\bm \Gamma} \mB^{T} \dualslp) \nonumber \\ 
\dualslp\!-\!2(\mathbf{I}\!+\!{\bf \Lambda}  \partial g^{*})^{-1}   {\bf \Lambda}  \mB \primslp& = (\mathbf{I}\!+\!{\bf \Lambda}  \partial g^{*})^{-1}(\dualslp\!-\!{\bf \Lambda}\mB\primslp).\label{equ_condition_fix_point}
\end{align}  

The fixed point characterization \eqref{equ_condition_fix_point} of nLasso solutions 
suggests the following coupled  fixed-point iterations: 
\begin{align}
\label{equ_fixed_point_iterations}  
\hat{\vy}^{(k+1)} &\defeq (\mathbf{I} + {\bf \Lambda}  \partial g^{*})^{-1}(\hat{\vy}^{(k)} +  {\bf \Lambda}  \mB(2\hat{\vx}^{(k)}- \hat{\vx}^{(k-1)}))\nonumber \\  
\hat{\vx}^{(k+1)} &\defeq (\mathbf{I} + {\bm \Gamma} \partial h)^{-1} (\hat{\vx}^{(k)} - {\bm \Gamma} \mB^{T} \hat{\vy}^{(k+1)}).
\end{align}  
The fixed-point iterations \eqref{equ_fixed_point_iterations} are obtained as a special case of the iterations 
\cite[Eq. (4)]{PrecPockChambolle2011} when choosing $\theta\!=\!1$ (using the notation in \cite{PrecPockChambolle2011}). 

The updates in \eqref{equ_fixed_point_iterations} allow for simple closed-form expressions
(see \cite[Sec. 6.2.]{pock_chambolle} for more details). 
Inserting these expressions into \eqref{equ_fixed_point_iterations} yields the updates 
\eqref{equ_pd_first_update}-\eqref{equ_pd_last_update} for iterative solving of nLasso \eqref{equ_nLasso}.

\noindent{\bf Bounding Sub-Optimality.} The identity \eqref{equ_zero_duality_gap} allows to 
bound the sub-optimality $\nLassoObj\big(\tilde{\vx}\big) - \nLassoObj\big(\hat{\vx}\big)$ 
of a given candidate $\tilde{\vx}$ for the solution of nLasso \eqref{equ_nLasso_unconstr}. Inserting 
an arbitrary dual vector $\vy$ into  \eqref{equ_zero_duality_gap}, 
\begin{equation}
\label{equ_upper_bound_subopt}
\nLassoObj\big( \tilde{\vx} \big)-  \nLassoObj\big(\hat{\vx}\big) \leq \nLassoObj\big( \tilde{\vx} \big) - \nDualLassoObj(\vy). 
\end{equation} 
Note that the right hand side in \eqref{equ_upper_bound_subopt} can be evaluated for any given pair $\hat{\vx}$, $\vy$ of primal and dual vectors.

\noindent{\bf Stopping Criteria.}
Possible stopping criteria for the updates \eqref{equ_pd_first_update}-\eqref{equ_pd_last_update} include a fixed 
number of iterations or testing for a sufficiently small sub-optimality gap $\mathcal{L}(\bar{x}^{(k)}) -\mathcal{L}(\hat{x})$. 
We can ensure a maximum sub-optimality gap using the bound \eqref{equ_upper_bound_subopt}. 
When using a fixed number $K$ of iterations, one can use well-known results on 
the convergence rate of primal-dual methods \cite{pock_chambolle}. Roughly speaking 
these results imply that the sub-optimality of the iterates $\bar{x}^{(K)})$ decrease 
according to $\propto 1/K$.  
The convergence rate $1/K$ is tight among all message passing methods to solve \eqref{equ_nLasso}. 
It is attained in chain-structured graphs (see \cite{ComplexitySLP2018}). 


\noindent{\bf Proof of Proposition \ref{prop_dual_TV_min_flow}}
We first note that \eqref{equ_dual_SLP} is equivalent to 
\begin{align}
\label{equ_network_flow_intmdt} 
\min_{\widehat{\vy} \in \mathbb{R}^{\edges}} & \sum_{i \in \trainingset}  v_{i} \big( (1/2) v_{i} \!-\!x_{i} \big),   \\ 
\mbox{s.t.} &  \sum_{j \in \mathcal{N}^{+}(i)} \hat{y}_{(i,j)} -  \sum_{j \in \mathcal{N}^{-}(i)} \hat{y}_{(j,i)}   = \begin{cases} v_{i} &\mbox{ for } i \in \samplingset   \\ 0 & \mbox{ otherwise.} \end{cases} \label{equ_constr_first_prob_proof} \\ 
&  |\hat{y}_{(i,j)} | \leq \lambda W_{e} \mbox{ for all } e \in \edges. \nonumber
\end{align}
The definition \eqref{equ_conv_conj_h} and \eqref{equ_conv_conj_g} for the components of 
\eqref{equ_dual_SLP} enforce implicit constraints on the dual vector that are identical 
with the constrains \eqref{equ_constr_first_prob_proof}.  
Thus, any optimal dual vector $\widehat{\vy}$ solving \eqref{equ_dual_SLP},  
must satisfy the constraints \eqref{equ_constr_first_prob_proof}). 
However, the objective functions in \eqref{equ_network_flow_intmdt} and \eqref{equ_dual_SLP} 
coincide when evaluated for vectors $\vy$ satisfying \eqref{equ_network_flow_intmdt}. 
	
The final step of the proof is to verify equivalence of the optimization 
problems \eqref{equ_network_flow_intmdt} and \eqref{equ_network_flow}. 
To this end, we note that the additional edges $(i,\star)$ in $\widetilde{\graph}$ (see Fig.\ \ref{fig:duality}-(b)) 
have no capacity constraints, or ``unbounded'' capacity, which allows them to 
``discharge'' the node demands $v_{i}$, for $i \in \samplingset$.

Consider an optimal flow $\hat{y}$ which solves \eqref{equ_network_flow_intmdt}. 
We then construct a flow $\tilde{y}_{e}$ on the extended graph $\extgraph$ by setting 
$\tilde{y}_{e} \defeq \hat{y}_{e}$ for all intra-cluster edges $e \in \edges \setminus \partial \partition$ 
and $\tilde{y}_{(i,\star)} \defeq v_{i}$ for all sampled nodes $i \in \samplingset$. 

The accumulating 
node ``$\star$'' has only inward edges resulting in the total demand $\sum_{i \in \samplingset} v_{i}$. 
However, this sum is zero since the demands $v_{i}$ of the flow in \eqref{equ_network_flow_intmdt} 
must sum to zero (see \cite[Chap.\ 1]{BertsekasNetworkOpt}).




\end{document}